\documentclass[twoside]{article}

\usepackage[accepted]{aistats2021}
\usepackage[round]{natbib}
\renewcommand{\bibname}{References}
\renewcommand{\bibsection}{\subsubsection*{\bibname}}

\usepackage{hyperref}
\usepackage{url}
\usepackage{booktabs}       
\usepackage{amsfonts}       
\usepackage{nicefrac}       
\usepackage{microtype}      
\usepackage{microtype}
\usepackage{graphicx}
\usepackage{booktabs} 
\usepackage{amsmath}
\usepackage{amsthm}
\usepackage{amssymb}
\usepackage{array}
\usepackage{enumerate}
\usepackage{titlesec}
\usepackage{hyperref}
\usepackage{xcolor}
\usepackage{textpos}
\usepackage{mathtools}
\usepackage{algorithm}
\usepackage{algpseudocode}
\usepackage{float}
\usepackage{subcaption}
\usepackage{csvsimple}
\usepackage{booktabs}

\newcommand{\T}{\intercal}
\newcommand{\R}{\mathbb{R}}
\newcommand{\Df}{D^{\mathrm{full}}{}}
\newcommand{\Dk}{D^{\setminus k}{}}
\newcommand{\Dl}{D^{\setminus 1}{}}
\newcommand{\Xk}{X^{\setminus k}{}}
\newcommand{\Yk}{Y^{\setminus k}{}}
\newcommand{\tf}{\theta^{\mathrm{full}}{}}
\newcommand{\tk}{\theta^{\setminus k}{}}
\newcommand{\tl}{\theta^{\setminus 1}{}}
\newcommand{\tr}{\theta^{\mathrm{res}}{}}
\newcommand{\ti}{\theta^{\mathrm{inf}}}
\newcommand{\yk}{\hat{y}^{\setminus k}}
\newcommand{\zk}{\hat{z}^{\setminus k}}
\newcommand{\Lf}{L^{\mathrm{full}}{}}
\newcommand{\Lk}{L^{\setminus k}{}}

\newcommand{\deq}{\stackrel{\mathclap{\Delta}}{=}}
\newcommand{\iid}{\stackrel{\tiny{\mathrm{i.i.d.}}}{\sim}}
\newcommand{\am}{\mathrm{argmin}}
\newcommand{\argmint}{\underset{\theta}{\am}}

\newcommand{\yhat}{\hat{y}}
\newcommand{\blockcomment}[1]{}

\newtheorem{theorem}{Theorem}
\newtheorem{proposition}[theorem]{Proposition}
\newtheorem{corollary}[theorem]{Corollary}

\theoremstyle{definition}
\newtheorem{definition}[theorem]{Definition}

\DeclareMathOperator*{\argmin}{\arg\!\min}

\begin{document}

\twocolumn[

\aistatstitle{Approximate Data Deletion from Machine Learning Models}

\aistatsauthor{ Zachary Izzo \And Mary Anne Smart \And Kamalika Chaudhuri \And James Zou }

\aistatsaddress{ Dept. of Mathematics \\ Stanford University \\ \texttt{zizzo@stanford.edu} \And  Department of CS\&E \\ UC San Diego \\ \texttt{msmart@eng.ucsd.edu} \And Department of CS\&E \\ UC San Diego \\ \texttt{kamalika@cs.ucsd.edu} \And Deptartment of BDS \\ Stanford University \\ \texttt{jamesz@stanford.edu} } ]

\begin{abstract}
Deleting data from a trained machine learning (ML) model is a critical task in many applications. For example, we may want to remove the influence of training points that might be out of date or outliers. Regulations such as EU's General Data Protection Regulation also stipulate that individuals can request to have their data deleted. The naive approach to data deletion is to retrain the ML model on the remaining data, but this is too time consuming.
In this work, we propose a new approximate deletion method for linear and logistic models whose computational cost is linear in the the feature dimension $d$ and independent of the number of training data $n$. This is a significant gain over all existing methods, which all have superlinear time dependence on the dimension. We also develop a new feature-injection test to evaluate the thoroughness of data deletion from ML models.
\end{abstract}

\section{Introduction}
\label{intro}
Given a trained machine learning (ML) model, there are many settings where we would like to \emph{delete} specific training points from this trained model. Deletion here means that we need to post-process the model to remove the effect of the specified training point(s). One example of the need for deletion is the Right to be Forgotten requirement which is a part of many policies including the EU's General Data Protection Regulation and the recent California Consumer Privacy Act. The Right to be Forgotten stipulates that individuals can request to have their personal data be deleted and cease to be used by organizations and companies such as Google, Facebook, etc. The challenge here is that even after an organization deletes the data associated with a given individual, information about that individual may persist in predictions made by machine learning models trained on the deleted data. These predictions may in turn leak information, impeding the individual's ability to truly be ``forgotten." For example, recent works show how one can reconstruct training data by attacking vision and NLP models~\citep{zhang2019secret}. Therefore there is a great need for approaches to remove an individual's data from the trained ML model as much as possible.

We propose a computational model inspired by this problem. After allowing a reasonable amount of precomputation, the model designer will receive a request to delete a batch of $k$ points from the model. Our goal is to accomplish this task as efficiently and accurately as possible.

A first plausible solution is exact data deletion, where the goal is to exactly reproduce the model that would have been output had the deleted points been omitted from training. However, in general this is computationally demanding: except for a few limited scenarios (e.g. \citep{forget} for K-means clustering), it will require retraining the model from scratch. Even in the simple case of training a logistic regression model via SGD, this will take time $O(ndP)$, where $n$ is the size of the dataset, $d$ is the dimensionality of the data, and $P$ is the number of passes over the data. When deletion requests need to be fulfilled promptly and in an online setting, retraining the model completely is infeasible. This motivates our study of approximate data deletion: by relaxing the requirements for removing data from the model, we hope to make the problem computationally tractable.

Approximate data deletion has two main challenges -- algorithmic (i.e. how to delete points effectively and quickly) and evaluation (i.e. how to quantify the quality of our approximate deletion). In this paper, we make progress in both of these areas.

Existing approaches to approximate deletion include the use of influence functions \citep{blackbox} and Newton's method. Both of these methods have computational costs which scale as $\Omega(d^2)$, where $d$ is the dimensionality of the data. We develop the first approximate deletion method with $O(d)$ computational cost, dubbed the \emph{projective residual update}, which computes the projection of the exact parameter update vector onto a particular low-dimensional subspace. The dependence of the computational cost on $d$ matches the trivial lower bound $\Omega(d)$ required to fully specify all of the entries of the model parameters (which we also assume to be $d$-dimensional), and is independent of the number of training data $n$. We additionally show that the PRU is optimal in terms of deletion accuracy within a certain class of gradient-based deletion methods.

Additionally motivated by privacy concerns, we propose a new evaluation criterion, dubbed the \emph{feature injection test}, which captures a deletion method’s ability to remove the model’s knowledge of sensitive attributes of the deleted points. The test works by adding a synthetic feature to only the deleted points which is perfectly correlated with the label, then measuring the amount by which the deletion method removes the weight on this artificial feature. All of our theoretical findings are corroborated with experiments on both real and synthetic datasets.
  
\textbf{Summary of contributions.}
\begin{itemize}
    \item We introduce a novel approximate data deletion method, the \emph{projective residual update} (PRU), which has a time complexity that is \emph{linear} in the dimension of the deleted data and is independent of the size of the dataset. We show that this method is optimal among a certain class of gradient-based updates in terms of deletion accuracy.
    
    \item We propose a new metric for evaluating data removal from models---the \emph{feature injection test} (FIT)---which captures how well we can remove the model's ``knowledge" of a sensitive, highly predictive feature present in the data.
    
    \item Experiments support our theoretical findings.
\end{itemize}

\section{Notation and Problem Setup}
\label{notation}
For the reader's convenience, we collect key notation and background here.
Throughout the paper, $n$ denotes the total number of training points, $d$ denotes the data dimension, and $k$ denotes the number of data points to be deleted from the model.
The $k$ points to be deleted will be supplied as a batch request---that is, the $k$ points should be deleted simultaneously, rather than one-by-one. We may think of this either as a request from a group of individuals, or a request to delete all of the data for one individual who has $k$ datapoints associated to her in the database.
We will always assume that $n\gg d \gg k$.

\begin{itemize}
    \item $\theta\in\R^d$ denotes the model parameters.
    \item $\Df = \{(x_i, y_i)\}_{i=1}^n \subseteq \R^d\times \R$ is the full set of training data. Throughout the paper, we assume that the feature vectors $x_i$ are in general position, i.e. that any collection of at most $d$ $x_i$s may be assumed to be linearly independent. This assumption holds with probability 1 when the $x_i$ are drawn i.i.d. from any distribution arising from a probability density on $\R^d$ (i.e. a probability distribution on $\R^d$ which is absolutely continuous with respect to the Lebesgue measure), for instance a non-degenerate Gaussian.
    \item $X = \begin{bmatrix} x_1 & \cdots & x_n\end{bmatrix}^\T \in \R^{n\times d}$ is the data matrix for $\Df$; its rows are the feature vectors $x_i^\T$. Note that since we have assumed that the $x_i$ are in general position and that $n\gg d$, $X$ is implicitly assumed to have full column rank.
    \item $Y = (y_1, \ldots, y_n)^\T \in \R^n$ is the response vector for $\Df$.
    \item $\Dk = \{(x_i, y_i)\}_{i=k+1}^n$ is the dataset with the $k$ desired points removed. We assume WLOG that that these are the first $k$ points, and we will frequently refer to this as the leave-$k$-out (LKO) dataset.
    \item $\Lf(\theta) = \sum_{i=1}^n \ell(x_i, y_i; \theta) + \frac{\lambda}{2} \lVert \theta \rVert_2^2$ is the (ridge-regularized) loss on the full dataset. The ``single-point" loss function $\ell$ will be the quadratic loss for linear regression ($\frac12 (\theta^\T x_i - y_i)^2$). Note that this includes the unregularized setting by simply taking the regularization strength $\lambda = 0$.
    \item $\Lk(\theta) = \sum_{i=k+1}^n \ell(x_i, y_i; \theta) + \frac{\lambda}{2}\lVert \theta \rVert_2^2$ is the loss on the LKO dataset. We require that the regularization strength be fixed independent of the number of samples.
    \item $\tf = \argmin_\theta \Lf(\theta)$ are the model parameters when fitted to the full dataset. We will  refer to the model with these parameters as the full model. In the case of linear regression, $\tf$ has the explicit form $\tf = (X^\T X + \lambda I)^{-1} XY$. (This is derived by setting the gradient of the loss to zero.)
    \item $\tk = \argmin_\theta \Lk(\theta)$ are the model paramteres when fitted to the LKO dataset. We will refer to the model with these parameters as the LKO model.
    \item $\yk_i = \tk^\T x_i$ is the prediction of the LKO model on the $i$-th datapoint.
\end{itemize}
We remark that although the number of data points $k$ to be deleted from the model is small compared to the dimension $d$, we make no more assumptions. In particular, we do not assume that the removed points need to be in any way ``similar" (e.g. i.i.d.) to the rest of the data and specifically consider cases where large outliers are removed. As a result, the removal of these points can still have a large impact on the model parameters.

For our discussions of computational cost, we are interested in updating and quickly redeploying the model after a deletion request. Thus, we consider only the ``just-in-time" (i.e. at the time of the deletion request) computational cost of each method. A reasonable amount of precomputation (that is, computations which can be done without knowledge of the $k$ points to be deleted) will be permitted without being included in the computational cost. Here, ``reasonable" is simply meant to exclude trivial but prohibitively expensive methods such as training a model on each subset of the training data, then returning the model parameters corresponding to the dataset with the appropriate points removed at deletion time.

We emphasize that we will focus on fulfulling a \emph{single} such batch deletion request. While simple, this framework captures the key essence of the data deletion challenge. Extending our methods to work in a fully online setting, where we may receive several batch deletion requests and the precomputation required between each request becomes significant, is an important next step towards practical approximate deletion methods.


We obtain results for both linear and logistic regression models. The results for logistic regression are an extension of the results for linear regression, so we choose to focus primarily on linear regression in the main body of the paper and defer a more complete discussion of logistic regression to the appendix.

Finally, we note that while we focus on linear models for the sake of theoretical clarity, these two scenarios capture most of the difficulty for nonlinear models as well. When retraining e.g. deep neural networks, it is often sufficient to consider all but the final layer as a fixed feature map on top of which we perform either linear or logistic regression (for regression and classification tasks, respectively). Retraining only the last layer is then sufficient and reduces to the two cases we consider in this paper. This method can be seen in \citep{blackbox}, in which the authors retrain their model to determine which training images are most influential for an image classification task; and in \citep{shapley}, where the authors retrain their model in order to compute data Shapley values, a measure of how much each data point contributes to the model's overall accuracy. In both cases, retraining only the last layer of the model is sufficient to give accurate results, and our results here can be similarly applied to the last layer.

\section{Methods}
\label{methods}
We give a brief overview of approximate deletion methods for parametric models from the literature.
\paragraph{Exact retraining}
The most straightforward way to remove data is by retraining the model completely. For the case of linear regression, we can naively compute $\tk$ using the analytic formula $\tk = (\Xk^\T \Xk + \lambda I)^{-1}\Xk^\T \Yk$. ($\Xk$ and $\Yk$ are the data matrix and response vector for $\Dk$, respectively.) The bottleneck is in forming the new Hessian $\Xk^\T \Xk$, giving an overall computational cost of $O(nd^2)$. Alternatively, we could retrain via iterative methods like SGD. This will take time $O(ndP)$, where $P$ is the number of passes over the dataset.
\paragraph{Newton's method}
Recent work \citep{certified} has attempted approximate retraining by taking a single step of Newton's method. This amounts to forming a quadratic approximation to the LKO loss $\Lk$ and moving to the minimizer of the approximation. This can be done in closed form, yielding the update
\begin{equation}
    \theta_{\textrm{Newton}} = \tf - [\nabla_\theta^2 \Lk(\tf)]^{-1}\nabla_\theta\Lk(\tf).
\end{equation}
When the loss function is quadratic in $\theta$ (as is the case in least squares linear regression), the approximation to $\Lk$ is just $\Lk$ itself and so Newton's method gives the exact solution. That is, in the case of linear regression, Newton's method reduces to the trivial ``approximate" retraining method of retraining the model exactly.

Since the full Hessian can be computed without knowing which points need to be deleted, we can consider it an offline cost. For linear regression, the new Hessian matrix is a rank $k$ update of the full Hessian, which can be computed via the Sherman-Morrison-Woodbury formula in $O(kd^2)$ time. In general, forming and inverting the new Hessian may take up to $O(nd^2)$ time.

\paragraph{Influence method} 
Recent works studied how to estimate the influence of a particular training point on the model's predictions \citep{ij, blackbox}. While the original methods were developed for different applications---e.g. interpretation and cross-validation---they can be adapted to perform approximate data deletion. Under suitable assumptions on the loss function $\ell$, we can view the model parameters $\theta$ as a function of weights on the data: $\theta(w) \equiv \argmin_\theta \sum_{i=1}^n w_i \ell(x_i, y_i; \theta)$.
In this setting, $\tf = \theta(\textbf{1})$ where $\textbf{1}$ is the all $1$s vector and $\tk = \theta((\underbrace{0,\ldots}_k,\underbrace{1,\ldots}_{n-k})^\T)$. The influence function approach (henceforth referred to as the influence method) uses the linear approximation to $\theta(w)$ about $w=\textbf{1}$ to estimate $\tk$. \citep{ij, blackbox} show that the linear approximation is given by
\begin{equation}
    \theta^{\textrm{inf}} = \tf - [\nabla_\theta^2 \Lf(\tf)]^{-1} \nabla_\theta \Lk(\tf).
\end{equation}
Assuming that we already have access to the inverse of the  Hessian, the bottleneck for this method is the Hessian-gradient product. This gives a $O(d^2)$ computational cost.

We summarize the asymptotic online computational costs in Table \ref{asymptotic_runtimes} alongside the computational cost of our novel method, the projective residual update. (Since the Newton step with Sherman-Morrison formula is exact and has a strictly lower computational cost than the naive method of retraining from scratch, we do not include the naive method in the table.) The precomputation costs for each of the methods (Newton, influence, and PRU) are approximately the same; they are dominated by forming and inverting the full Hessian, which takes time $O(nd^2)$.
\begin{table}[h]
\caption{Asymptotic computational costs for each approximate retraining method. The projective residual update is the only method with linear dependence on $d$.}
\label{asymptotic_runtimes}
\begin{center}
\begin{small}
\begin{sc}
\begin{tabular}{ccc}
\toprule
Exact &
Influence &
Projective residual \\
\midrule
$O(kd^2)$ &
$O(d^2)$  &
$O(k^2d)$ \\
\bottomrule
\end{tabular}
\end{sc}
\end{small}
\end{center}
\end{table}

\section{The Projective Residual Update}
We now introduce our proposed approximate update. We leverage \emph{synthetic data}, a term we use to refer to artificial datapoints which we construct and whose properties form the basis of the intuition for our method. We combine gradient methods with synthetic data to achieve an approximate parameter update which is fast for deleting small groups of points.
The intuition is as follows: if we can calculate the values $\yk_i = \tk^\T x_i$ that the model would predict on each of the removed $x_i$s \emph{without knowing $\tk$}, then minimize the loss of the model on the synthetic points $(x_i, \yk_i)$ for $i=1,\ldots, k$, we should expect our parameters to move closer to $\tk$ since $\tk$ achieves the minimum loss on the points $(x_i, \yk_i)$. We will minimize the loss on these synthetic points by taking a (slightly modified) gradient step.

It may be surprising that we can calculate the values $\yk_i$ without needing to know $\tk$. We accomplish this by generalizing a well-known technique from statistics for computing leave-one-out residuals for linear models.
As in the influence function applications, we incur an upfront cost of forming the so-called ``hat matrix" $H\equiv X(X^\T X + \lambda I)^{-1}X^\T$ for the full linear regression. Since we can compute this matrix without needing to know which points will be deleted, it is reasonable to consider it as an offline computation which will not be included in the computational cost of the update itself.

We formalize the intuition for the update as follows. Assume that we can compute $\yk_i$ efficiently, without needing to know $\tk$. The gradient of the loss on the synthetic points $(x_i, \yk_i)$ is $\nabla_\theta L^{\{(x_i, \yk_i)\}_{i=1}^k}(\theta) = \sum_{i=1}^k (\theta^\T x_i - \yk_i)x_i$. Substituting $\tk^\T x_i$ for $\yk_i$ and rearranging, then setting $\theta = \tf$ shows that $\nabla_\theta L^{\{(x_i, \yk_i)\}_{i=1}^k}(\tf) = \left(\sum_{i=1}^k x_ix_i^\T \right)(\tf - \tk)$.
We show that the form of the matrix $\sum_{i=1}^k x_ix_i^\T$ allows us to efficiently compute a pseudoinverse. We summarize these steps in Algorithm \ref{alg:res}.

\begin{algorithm}[H]
\caption{The projective residual update}\label{alg:res}
\begin{algorithmic}[1]
\Procedure{PRU}{$X, Y, H, \tf, k$}
\State $\hat{y}'_1,\ldots,\hat{y}'_k \gets$ \textsc{LKO}($X, Y, H, k$)
\State $S^{-1} \gets$ \textsc{PseudoInv}$(\sum_{i=1}^k x_i x_i^\T)$
\State $\nabla L \gets \sum_{i=1}^k (\tf^\T x_i - \hat{y}'_i)x_i$ 
\State \textbf{return} $\tf - \textsc{FastMult}(S^{-1}, \nabla L)$
\EndProcedure
\end{algorithmic}
\end{algorithm}
\begin{algorithm}[H]
\caption{Leave-$k$-out predictions}\label{alg: lko}
\begin{algorithmic}[1]
\Procedure{LKO}{$X, Y, H, \tf, k$}
\State $R \gets Y_{1:k} - X_{1:k}\tf$
\State $D \gets \textrm{diag}(\{(1-H_{ii})^{-1}\}_{i=1}^k)$
\State $T_{ij} \gets \textbf{1}\{i\neq j\}\frac{H_{ij}}{1-H_{jj}}$
\State $T \gets (T_{ij})_{i,j=1}^k$
\State $\hat{Y}^{\setminus k} \gets Y_{1:k} - (I-T)^{-1}DR$
\State \textbf{return} $\hat{Y}^{\setminus k}$
\EndProcedure
\end{algorithmic}
\end{algorithm}


The results of running the residual update are described by Theorem \ref{main}, our main theorem.
\begin{theorem}\label{main}
Algorithm \ref{alg:res} computes $\tr = \tf + \textrm{proj}_{\textrm{span}(x_1,\ldots,x_k)}(\tk - \tf)$ with computational cost $O(k^2d)$.
\end{theorem}
The result of Theorem \ref{main} is striking. It says that the projective residual update makes the \emph{most improvement possible} for any parameter update which is a linear combination of the removed $x_i$s. As a direct result of this, we have the following corollary.

\begin{definition}
For any dataset $D = \{(\bar{x}_i, \bar{y}_i)\}_{i=1}^N$ (not necessarily the same as the original dataset $\Df$), define $L^D(\theta) = \sum_{i=1}^N\frac12 (\theta^\T \bar{x}_i - \bar{y}_i)^2$. Define a \emph{gradient-based update} of the model parameters $\tf$ as any update $\theta^{\mathrm{approx}}$ which can be computed by the following procedure: set $\theta_0 = \tf$, then define $$\theta_{t+1} = \theta_t - \alpha_t \nabla_\theta L^{D_t}(\theta_t)$$ for some sequence of datasets $D_t$. Finally, let $\theta^{\mathrm{approx}} = \theta_T$ for some $T$.
\end{definition}

\begin{corollary} \label{cor: opt}
Let $\mathcal{S} = \{x_i\}_{i=1}^k \times \R$ be the set of all datapoints whose feature vectors belong to the set of points to be deleted from the original dataset $\Df$. If $\theta^{\mathrm{approx}}$ is any gradient-based update of $\tf$ with $D_t \subseteq \mathcal{S}$ for all $t$, then we have
$$\lVert \tk - \tr \rVert \leq \lVert \tk - \theta^{\mathrm{approx}}\rVert.$$
\end{corollary}

\begin{proof}
This follows immediately from the fact that the gradient of the square loss on any point $(x_i, y)$ is a scalar multiple of $x_i$, and therefore the update $\tf - \theta^{\mathrm{approx}} \in \mathrm{span}(x_1,\ldots,x_k)$. 
\end{proof}


As we will see later, Theorem \ref{main} guarantees that PRU performs well for deleting the model's knowledge of sensitive attributes under data sparsity conditions.

\paragraph{The \textsc{LKO}, \textsc{PseudoInv}, and \textsc{FastMult} subroutines}
The ability to efficiently calculate $\yk_i$, $i=1,\ldots k$ is crucial to our method. Algorithm \ref{alg: lko} accomplishes this by generalizing a well-known result from statistics which allows one to compute the leave-one-out residuals $\hat{y}^{\setminus 1}_i-y_i$. In Algorithm \ref{alg: lko}, $X_{1:k}$ denotes the first $k$ rows of $X$, $Y_{1:k}$ the first $k$ entries of $Y$, and $H$ the hat matrix for the data, defined below. Since the residuals $r_i = y_i - x_i^\T \tf$ and the hat matrix $H$ can be computed before the time of the deletion request, these steps can be excluded from the total computational cost of Algorithm \ref{alg: lko}.

\begin{theorem}\label{lko}
Algorithm \ref{alg: lko} computes the LKO predictions $\yk_i$, $i=1,\ldots,k \:$ in $O(k^3)$ time.
\end{theorem}
The proof of Theorem \ref{lko} can be found in Appendix \ref{proof: lko}.
The low-rank structure of $A \equiv \sum_{i=1}^k x_i x_i^\T$ allows us to quickly compute its pseudoinverse. We do this by finding the eigendecomposition of an associated $k\times k$ matrix (which can again be done quickly when $k$ is small, see e.g. \citep{eigen}), which we then leverage to find the eigendecomposition of $A$. Computing the pseudoinverse in this way also allows us to multiply by it quickly.
For a more detailed explanation, refer to the appendix.

\subsection{Outlier deletion} \label{performance}
To illustrate the usefulness of the residual update, we consider its performance compared to the influence method on the dataset $\Df = \{(\lambda x_1, \lambda y_1)\} \cup \Dl$, where $\Dl=\{(x_i, y_i)\}_{i=2}^{n+1}$ and we are attempting to remove the first datapoint from $\Df$ so that we are left with $\Dl$. In particular, we examine the difference in performance between the residual and influence updates as the parameter $\lambda\rightarrow\infty$.

\begin{theorem}\label{infbad}
Let $\Df = \{(\lambda x_1, \lambda y_1)\} \cup \Dl$. Then $\theta^{\mathrm{inf}} \rightarrow \tf$ as $\lambda\rightarrow\infty$.
\end{theorem}
Theorem \ref{infbad} says that when we try to delete points with large norm, the influence method will barely update the parameters at all, with the update shrinking as the size of the removed features increases. This makes intuitive sense. The performance of the influence method relies on the Hessian of the full loss being a good approximation of the Hessian of the leave-one-out loss. As the scaling factor $\lambda$ grows, the full Hessian $X^\T X + \lambda^2 x_1 x_1^\T$ deviates more and more from the LOO Hessian $X^\T X$, causing this drop in performance.
On the other hand, as the size of the outlier grows, the exact parameter update vector $\tf - \tl$ approaches a well-defined, finite limit. The PRU computes the projection of this update onto the subspace spanned by the deleted points, and therefore in general its improvement will remain bounded away from 0 even as the outlier grows. It follows that the PRU will outperform the influence method for the deletion of large enough outliers. For a complete proof of this fact, see Proposition \ref{asymptotic_baseline} in Appendix \ref{sec: outlier}.

\subsection{Extension to logistic regression} \label{sec: logistic pru}
The generalization of the PRU to logistic regression relies on the fact that a logistic model can be trained by iteratively reweighted least squares; indeed, a Newton step for logistic regression reduces to the solution of a weighted least squares problem \citep{murphy2012}. We leverage this fact along with the generalization of Theorem \ref{lko} from Appendix \ref{appendix: generalized lko} to compute a fast approximation to a Newton step. The method is given by Algorithm \ref{alg:logistic_pru}. (Note: $H_{\lambda, w}$ denotes the Hessian for a weighted linear least squares problem with weights $w$ and regularization $\lambda.$)

\begin{algorithm}
\caption{The PRU for logistic regression}\label{alg:logistic_pru}
\begin{algorithmic}[1]
\Procedure{LogisticPRU}{$X, Y, \tf, k$}
\For{$i=1,\ldots, n$}
\State $w_i \gets h_{\tf}(x_i)(1-h_{\tf}(x_i))$
\EndFor
\State $S_{\tf} \gets \textrm{diag}(w)$
\State $Z \gets X\tf + S_{\tf}^{-1}(Y - h_{\tf})$ 
\State \textbf{return} $\textsc{ResidualUpdate}(X, Z, H_{\lambda, w}, \tf, k)$
\EndProcedure
\end{algorithmic}
\end{algorithm}

\begin{theorem}\label{logisticresup}
Algorithm \ref{alg:logistic_pru} computes the update $\tr = \tf + \textrm{proj}_{\textrm{span}(x_1,\ldots,x_k)}(\Delta_{\textrm{Newton}})$ in $O(k^2 d)$ time.
\end{theorem}
Refer to Appendix \ref{gen_to_logistic} for an explanation of the algorithm and the proof.

\section{Evaluation Metrics}
\label{metrics}
\paragraph{\texorpdfstring{$L^2$}{L2} distance}
A natural way to measure the effectiveness of an approximate data deletion method is to consider the $L^2$ distance between the estimated parameters and the parameters obtained via retraining from scratch. If the approximately retrained parameters have a small $L^2$ distance to the exactly retrained parameters, then when the models depend continuously on their parameters (such as in linear regression), the models are guaranteed to make similar predictions.

In addition to the general similarity between two models captured by the $L^2$ distance, we are also interested in studying a more fine-grained metric: how well can an approximate deletion method remove specific sensitive attributes from the retrained model? This motivates a new metric that we propose: the \emph{feature injection test}.
\paragraph{Feature injection test}
The rationale behind this new test is as follows.
If a user's data belongs to some small minority group within a dataset, that user may be concerned about what the data collector will be able to learn about her and this small group. When she requests that her data be deleted from a model, she will want any of these localized correlations that the model learned to be forgotten. 

This thought experiment motivates a new test for evaluating data deletion, which we call the feature injection test (FIT). We inject a strong signal into our dataset which we expect the model to learn. Specifically, we append an extra feature to the data which is equal to zero for all but a small subset of the datapoints, and which is perfectly correlated with the label we wish to predict. In the case of a linear classifier, we expect the model to learn a weight for this special feature with absolute value significantly greater than zero. After this special subset is deleted, however, any strictly positive regularization will force the weight on this feature to be 0 in the exactly retrained model. We can plot the value of the model's learned weight for this special feature before and after deletion and use this as a measure of the effectiveness of the approximate deletion method.

Below we give a general description the FIT for logistic regression. Let $\Df = \{(x_i, y_i)\}_{i=1}^n \subseteq \R^d \times \{0, 1\}$ be the full dataset and assume WLOG that we wish to delete points $i=1,\ldots,k$. We require that the deleted points all belong to the positive class, i.e. $y_1=\cdots=y_k = 1$.
\begin{enumerate}
    \item Set $\tilde{x}_i = [x_i^\T, \: 1]^\T$ for $1\leq i \leq k$ and $\tilde{x}_i = [x_i^\T, \: 0]^\T$ for $k < i \leq n$. The last entry of each $\tilde{x}_i$ is the injected feature; each deleted point has an injected feature with value 1, while the non-deleted points have injected feature value 0.
    \item Train a logistic classifier on $\{(\tilde{x}_i,y_i)\}_{i=1}^n$ (using ridge-regularized  cross-entropy loss and strictly positive regularization strength) and let $\tf \in \R^{d+1}$ be the weights of the resulting model. Define $w_* = \tf[d+1]$ to be the $d+1$-th entry of $\tf$, i.e. the weight corresponding to the injected feature.
    \item Given the output $\theta^{\textrm{approx}}$ of an approximate retraining method, its FIT metric is defined as $\theta^{\textrm{approx}}[d+1] / w_*$. The closer the FIT metric is to 0, the better the approximate deletion method is at removing the injected sensitive feature from the model.
\end{enumerate}
For a description of the FIT for linear regression, see Appendix \ref{sec: data construction}.

\begin{table*}[t!]
\caption{Mean runtimes for each method as a fraction of full retraining runtime (100 trials). (\textsc{inf} stands for influence method.) In all instances, the standard error was not within the significant digits of the mean (all standard errors were of order $10^{-4}$ or smaller) so for clarity we do not include the errors. The absolute runtimes of the exact method to which we compare is in the appendix. The results match our theory that PRU's runtime is especially advantageous for high dimensions and relatively small $k$. } 
\label{runtime_experiment}
\begin{center}
\begin{small}
\begin{sc}
\begin{tabular}{l|ccccc}
\toprule
& $d = 1000$ & $d = 1500$ & $d = 2000$ & $d = 2500$ & $d = 3000$\\
\midrule
$k = 1$ (inf) &
$0.0085$ &
$0.0053$ &
$0.0041$ &
$0.0036$ &
$0.0028$ \\
$k = 1$ (pru) &
$\textbf{0.0062}$ &
$\textbf{0.0017}$ &
$\textbf{0.0008}$ &
$\textbf{0.0004}$ &
$\textbf{0.0003}$ \\
\midrule
$k = 5$ (inf) &
$\textbf{0.0092}$ &
$0.0052$ &
$0.0043$ &
$0.0033$ &
$0.0028$ \\
$k = 5$ (pru) &
$0.0112$ &
$\textbf{0.0035}$ &
$\textbf{0.0019}$ &
$\textbf{0.0011}$ &
$\textbf{0.0007}$ \\
\midrule
$k = 10$ (inf) &
$\textbf{0.0098}$ &
$0.0054$ &
$0.0045$ &
$0.0033$ &
$0.0031$ \\
$k = 10$ (pru) &
$0.0155$ &
$\textbf{0.0049}$ &
$\textbf{0.0025}$ &
$\textbf{0.0015}$ &
$\textbf{0.0010}$ \\
\midrule
$k = 25$ (inf) &
$\textbf{0.0105}$ &
$\textbf{0.0058}$ &
$\textbf{0.0050}$ &
$\textbf{0.0035}$ &
$0.0032$ \\
$k = 25$ (pru) &
$0.0365$ &
$0.0121$ &
$0.0067$ &
$0.0037$ &
$\textbf{0.0026}$ \\
\midrule
$k = 50$ (inf) &
$\textbf{0.0122}$ &
$\textbf{0.0065}$ &
$\textbf{0.0051}$ &
$\textbf{0.0036}$ &
$\textbf{0.0033}$ \\
$k = 50$ (pru) &
$0.0794$ &
$0.0273$ &
$0.0151$ &
$0.0085$ &
$0.0059$ \\
\bottomrule
\end{tabular}
\end{sc}
\end{small}
\end{center}
\end{table*}

\section{Empirical Validation}
We now verify our theoretical guarantees and compare the accuracy and speed of the various retraining methods experimentally. We emphasize that these experiments are intended to confirm the theory rather than demonstrate practical usage. Deploying and testing large-scale data deletion methods is an important direction of future work.
Our analysis and methods provide an important initial step towards this goal. Code for reproducing our experiments can be found at \href{https://github.com/zleizzo/datadeletion}{\texttt{https://github.com/zleizzo/datadeletion}}.
\subsection{Linear regression}
\paragraph{Synthetic datasets}
The synthetic datasets are constructed so that the linear regression model is well-specified. That is, given the data matrix $X$, the response vector $Y$ is given by $Y = X\theta^* + \varepsilon$, where $\varepsilon \sim N(0, \sigma^2 I_n)$ is the error vector. For all of the synthetic datasets, we take $n=10d$.
Slight modifications are made to this general setup for each experiment. For outlier removal tests, we scale a subset of the full dataset to create outliers, then delete these points. For the sparse data setting, we generate sparse feature vectors rather than drawing from a Gaussian. For full details on dataset construction, refer to the appendix.
\paragraph{Yelp} We select 200 users from the Yelp dataset and use their reviews (2100 reviews in total). We use a separate sample of reviews from the dataset to construct a vocabulary of the 1500 most common words; then we represent each review in our dataset as a vector of counts denoting how many times each word in the vocabulary appeared in the given review. Four and five star reviews are considered positive, and the rest are negative. To turn the regression model's predictions into a binary classifier, we threshold scores at zero--a predicted value that is greater than zero becomes a prediction of the positive class while a predicted value that is less than zero becomes a prediction of the negative class.
\paragraph{Results - Synthetic data}
The experimental results closely match the theory in all respects. For the runtime experiments, refer to Table \ref{runtime_experiment}. Both the influence method and the projective residual update are significantly faster than exact model retraining. In the extreme case of $d=3000$ and a removal group of size $k=1$, the projective residual update is more than 3000 times faster than exact retraining. The relative speed of the PRU and influence method are also as we expect: PRU is faster than influence for small group sizes, and the size of the largest group that we can delete while maintaining this speed advantage increases as $d$ increases. For 3000-dimensional data, PRU has the speed advantage for deleting groups as large as 25.

For the FIT, refer to Table \ref{synthetic_weight}. As the data matrix becomes more sparse, the span of the removed points become more likely to contain the $d$-th standard basis vector $e_d$ (or a vector very close to it), allowing the residual update to completely remove the special weight. We observe this phenomenon in several of the cases we tested (denoted by an asterisk in table \ref{synthetic_weight}).
\begin{table}[!htb]
      \caption{Mean results for the FIT on synthetic data (50 trials). The special weight is given as fraction of baseline weight (lower the better). Results are for $d=1500$ for various group sizes ($k$) and sparsity values ($p$). See text for discussion of the standard errors and the notable values (indicated by asterisks). The baseline weights to which we compare can be found in the appendix. These results match our theory that PRU performs especially well in the sparse regime.}
    \label{synthetic_weight}    
        \begin{center}
        \begin{small}
        \begin{sc}
        \begin{tabular}{l|cccc}
        \toprule
        & $p = 0.25$ & $0.1$ & $0.05$\\
        \midrule
        $k = 5$ (inf)&
        1.09&
        0.99&
        1.01\\
        $k = 5$ (pru)&
        \textbf{0.98}&
        \textbf{0.96}&
        \textbf{0.93}\\
        \midrule
        $k = 50$ (inf)&
        \textbf{0.84}&
        0.97&
        $2.32^{**}$\\
        $k = 50$ (pru)&
        0.86&
        \textbf{0.67}&
        \textbf{0.35}\\
        \midrule
        $k = 100$ (inf)&
        0.76&
        0.92&
        0.98\\
        $k = 100$ (pru)&
        \textbf{0.72}&
        \textbf{0.32}&
        $\textbf{0.00}^*$\\
        \bottomrule
        \end{tabular}
        \end{sc}
        \end{small}
        \end{center}
\end{table}
\begin{table}[!htb]
        
        \caption{Mean results for the $L^2$ test on synthetic data (50 trials). The $L^2$ distance is given as fraction of baseline distance ($\lVert \tf - \tk\rVert$; the values of the starting distance can be found in the appendix). Results are for $d=1500$ for various group sizes ($k$) and outlier sizes ($\lambda$, see Theorem \ref{infbad}).}
        \label{synthetic_l2}
        \begin{center}
        \begin{small}
        \begin{sc}
        \begin{tabular}{l|ccc}
        \toprule
        & $\lambda = 1$ & $\lambda = 10$ & $\lambda = 100$ \\
        \midrule
        $k = 5$ (inf) &
        $\textbf{0.38}$ &
        $0.93$ &
        $0.99$ \\
        $k = 5$ (pru) &
        $0.92$ &
        $\textbf{0.92}$ &
        $\textbf{0.92}$ \\
        \midrule
        $k = 50$ (inf) &
        $\textbf{0.16}$ &
        $0.91$ &
        $0.99$ \\
        $k = 50$ (pru) &
        $0.88$ &
        $\textbf{0.88}$ &
        $\textbf{0.88}$ \\
        \midrule
        $k = 100$ (inf) &
        $\textbf{0.14}$ &
        $0.90$ &
        $0.99$ \\
        $k = 100$ (pru) &
        $0.88$ &
        $\textbf{0.88}$ &
        $\textbf{0.88}$ \\
        \bottomrule
        \end{tabular}
        \end{sc}
        \end{small}
        \end{center}
\end{table}
All of the standard errors for the PRU were well below 5\% of the mean. In contrast, the influence method performs poorly compared to the PRU in most scenarios, in addition to exhibiting much less numerical stability.

For the $L^2$ metric, refer to Table \ref{synthetic_l2}. The influence method outperforms PRU for deleting ``typical" points (when $\lambda=1$, the deleted points are i.i.d. with the rest of the data rather than being outliers). As the size of the deleted points grows, however, we see a steep drop in the performance of the influence method, while PRU remains almost completely unaffected.
%
\paragraph{Results - Yelp}
Since the Yelp dataset does not have large outliers, the influence method outperforms the projective residual update in terms of $L^2$ distance. For larger groups, however, the PRU's performance on the FIT is superior to the influence method, which fails to remove the injected signal. These results are summarized in Figure \ref{fig:realdata}. The fact that the influence method performs well in terms of $L^2$ distance and yet poorly on the FIT for the same dataset highlights the fact that $L^2$ distance alone is not a sufficient metric to consider, especially if the main concern is privacy. Due to space constraints, the results for the $L^2$ test can be found in the appendix.
\begin{figure}[h]
\includegraphics[width=\linewidth]{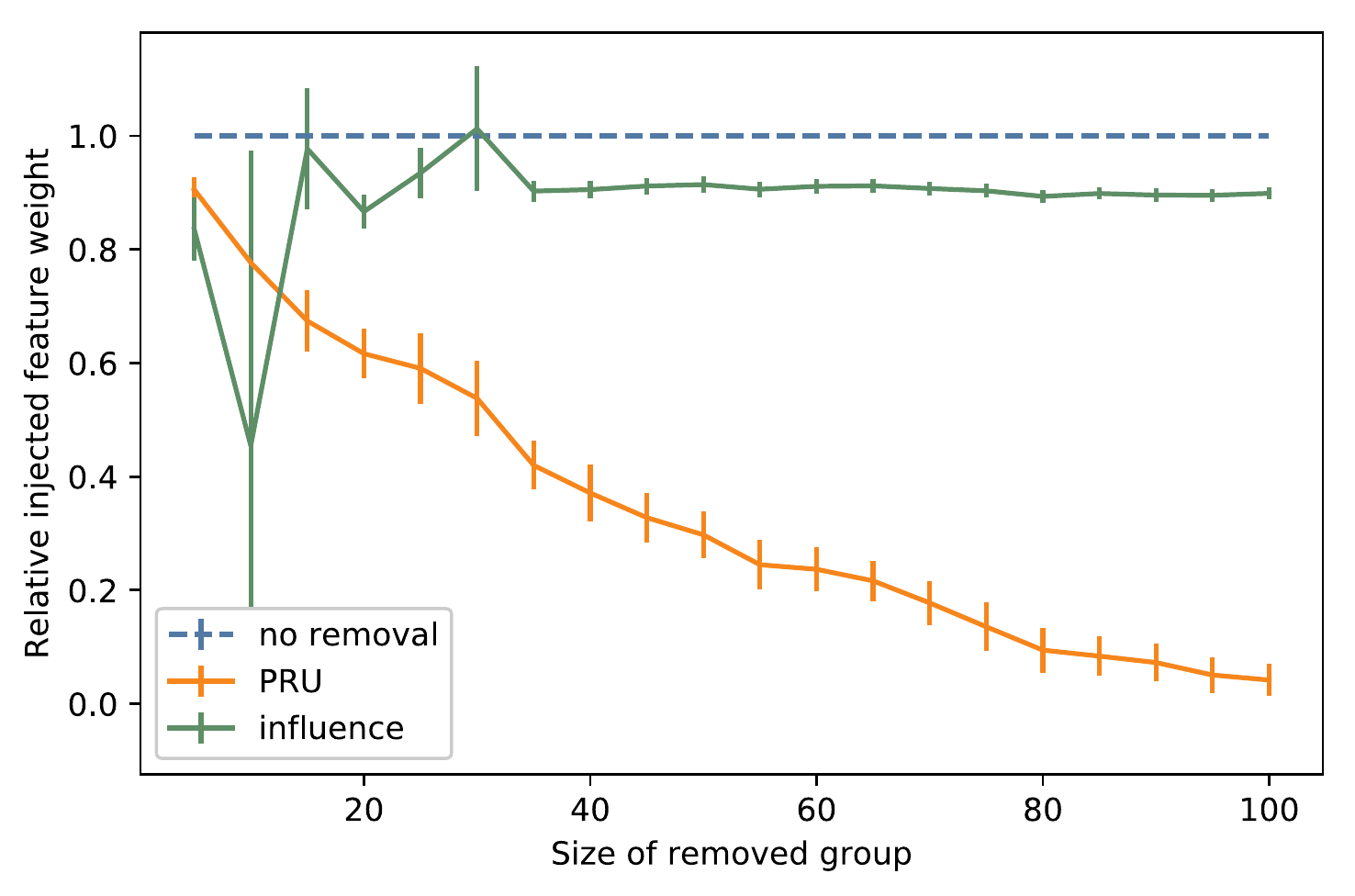}
\caption{
Yelp FIT experiment.
We plot the mean of each metric $\pm$ the standard error of the mean.
The PRU deletes the injected feature much more effectively and exhibits greater stability.}
\label{fig:realdata} 
\end{figure}
%
\subsection{Logistic regression}\label{sec: logistic regression}
We test the PRU on a synthetic data logistic regression experiment. The data $(x, y)\in \mathbb{R}^d \times \{0, 1\}$ were generated so that the logistic model is well-specified, i.e. there exists some parameter $\theta^*$ such that $\mathbb{P}(y = 1 | x) = \sigma(x^\T \theta^*)$, where $\sigma(z) = 1/(1 + e^{-z})$.
For this experiment, we generate $n=5000$ datapoints of dimension $d=1000$. We compare the influence method to the PRU and leave analysis of the Newton step to the appendix.
Refer to Tables \ref{logistic fit} and \ref{logistic l2}. Consistent with linear regression and our theory, in the sparse data regime, the PRU performs very well in terms of both the $L^2$ metric and the FIT. 

\begin{table}[!htb]
\caption{Median FIT results for logistic regression over 10 trials. Due to space constraints, we report these figures with the IQR in the the appendix; the variation across trials was generally very small. For larger group sizes and sparse data, the PRU is able to completely remove the injected feature.}
\label{logistic fit}
\begin{center}
\begin{small}
\begin{sc}
\begin{tabular}{l|ccc}
\toprule
             & $p=0.5$            & $0.1$            & $0.05$           \\
\midrule
\midrule
$k=25$ (inf) & \textbf{0.82} & 0.77 & 0.78 \\
$k=25$ (pru) & 0.86 & \textbf{0.69} & \textbf{0.44} \\
\midrule
$k=50$ (inf) & 0.81 & 0.82 & 0.82 \\
$k=50$ (pru) & 0.81 & \textbf{0.48} & \textbf{0.02} \\
\midrule
$k=100$ (inf) & 0.82 & 0.85 & 0.84 \\
$k=100$ (pru) & \textbf{0.71} & \textbf{0.00} & \textbf{0.00} \\
\bottomrule
\end{tabular}
\end{sc}
\end{small}
\end{center}
\hfill
\caption{Median $L^2$ results for logistic regression over 10 trials. See the appendix for IQR. We examine the performance of each method for different group deletion sizes $(k)$ and different levels of data sparsity ($p$). The results closely match the theory. For larger group and sparse data, PRU outperforms the influence method.}
\label{logistic l2}
\begin{center}
\begin{small}
\begin{sc}
\begin{tabular}{l|ccc}
\toprule
& $p=0.5$                    & $0.1$                    & $0.05$      \\
\midrule
$k=25$ (inf) & \textbf{0.85} & \textbf{0.77} & 0.78 \\
$k=25$ (pru) & 0.86 & 0.80 & \textbf{0.65} \\
\midrule
$k=50$ (inf) & 0.85 & 0.83 & 0.82 \\
$k=50$ (pru) & 0.85 & \textbf{0.69} & \textbf{0.20} \\
\midrule
$k=100$ (inf) & 0.85 & 0.86 & 0.84 \\
$k=100$ (pru) & \textbf{0.80} & \textbf{0.24} & \textbf{0.13} \\
\bottomrule
\end{tabular}
\end{sc}
\end{small}
\end{center}
\end{table}

\section{Related Work}
\label{related}
Most previous work on this topic has focused on specific classes of models. For example, Ginart et al. examined the problem of data deletion for clustering algorithms \citep{forget}. Tsai et al. use retraining with warm starts as a data deletion method
for logistic regression, although they refer to the problem as decremental training \citep{decr}. Others such as Cauwenberghs et al. have studied the problem of decremental training for SVM models \citep{svm}.
Cao et al. consider a more general class of models and propose a
solution using the statistical query framework for the problem of data deletion (which they refer to as machine unlearning);
their proposed method for adaptive SQ learning algorithms, such as gradient descent, is analogous to the aforementioned warm start method \citep{unlearn}. Bourtoule et al. introduce a method called SISA (Sharded, Isolated, Sliced, and Aggregated) training, that minimizes the computational cost of retraining by taking advantage of sharding and caching operations during training  \citep{bourtoule}.
Other previous approaches for machine unlearning are very closely related to the influence and Newton's methods. The method introduced by Monari and Dreyfus in \citep{Monari_Dreyfus_2000} is the same as the influence method with a different update step size. In the earlier work of Hansen and Larsen \citep{Hansen_Larsen_1996}, their proposed update is simply a Newton step.


While our work has applications to privacy, it is distinct from previous research focusing on privacy. The differential privacy framework, for instance, provides a way to minimize the risks associated with belonging to a model’s training set.
However, the strong privacy guarantees offered by differential privacy often come at the cost of significantly
reduced accuracy. In a setting where most users are not overly concerned about privacy and are willing
to share data, the option to use a non-private model
while allowing users to opt-out if they change their minds provides a useful middle ground.
Drawing on the definition of differential privacy, the authors of \citep{certified} define a notion of $\epsilon$-certified removal from machine learning models. They propose a modification of Newton's method for data deletion from linear models to satisfy this definition.

Our method's key advantage over previous work is that it is the first deletion algorithm for parametric models with a runtime linear in the data dimension and independent of the dataset size. This is a crucial development for modern high-dimensional ML.

\section{Conclusion}
\label{conclusion}
We consider the problem of approximate data deletion from ML models, with a particular focus on linear and logistic regression. We develop a novel algorithm---the projective residual update (PRU)---with a computational cost which is linear in the dimension of the data, a substantial improvement over existing methods with quadratic dimension dependence. We also introduce a new metric for evaluating data removal from models---the feature injection test---a measure of the removal of the model's knowledge of a sensitive, highly predictive feature present in the data.
Experiments on both real and synthetic data corroborate the theory. With any approximate deletion method, the accuracy of the approximation will decay as more deletion requests are processed. Extending our ideas to address this challenge is an important direction for future work.

\subsubsection*{Acknowledgements}
JZ is supported by NSF CCF 1763191, NSF CAREER 1942926, NIH P30AG059307, NIH U01MH098953 and grants from the Silicon Valley Foundation and the Chan-Zuckerberg Initiative. MS was supported by a Qualcomm Fellowship. KC thanks ONR under N00014-20-1-2334 and a Google Faculty Fellowship for research support. We also thank the anonymous reviewers for their insightful comments.

\nocite{*}
\bibliographystyle{plainnat}
\bibliography{theory}

\clearpage
\onecolumn
\appendix

\section{Detailed algorithm description and proof of Theorem \ref{main}}
\begin{algorithm}
\caption{Compute the pseudoinverse of $\sum_{i=1}^k x_i x_i^\T$\label{alg:pinv}}
\begin{algorithmic}[1]
\Procedure{PseudoInv}{$x_1, \ldots, x_k$}
\State $c_1,\ldots,c_k, u_1,\ldots,u_k \gets$ \textsc{Gram-Schmidt}($x_1,\ldots,x_k$)\Comment{$x_i = \sum_{j=1}^k c_{ij} u_j$.}
\State $C \gets \sum_{i=1}^k c_i c_i^\T$
\State $\lambda_1, a_1, \ldots, \lambda_k, a_k \gets$ \textsc{Eigendecompose($C$)}\Comment{Eigenvalue $\lambda_i$ has corresponding eignevector $a_i$}
\For{$i=1,\ldots,k$}
\State $v_i \gets \sum_{j=1}^k a_{ij}u_j$
\EndFor
\State \textbf{return} $\lambda_1^{-1}, v_1, \ldots, \lambda_k^{-1}, v_k$
\EndProcedure
\end{algorithmic}
\end{algorithm}

\begin{algorithm}
\caption{Fast multiplication by pseudoinverse}\label{alg:mastmult}
\begin{algorithmic}[1]
\Procedure{FastMult}{$S^{-1}, \nabla L$} \Comment{$S^{-1}$ must be given in its low-rank form $S^{-1} = \sum_{i=1}^k \lambda_i^{-1} v_i v_i^\T$}
\State \textbf{return} $\sum_{i=1}^k (\lambda_i^{-1} (v_i^\T \nabla L))v_i$ \Comment{Compute according to the specified parenthesization}
\EndProcedure
\end{algorithmic}
\end{algorithm}

We take a gradient step in the direction specified by the synthetic LKO points $(x_i, \yk_i)$, $i=1,\ldots, k$. That is, we update
\begin{equation}\label{grad}
\tr = \tf - \alpha\sum_{i=1}^k (\tf^\T x_i - \yk_i)x_i.
\end{equation}
Ordinarily, the parameter $\alpha$ is a scalar which specifies the step size. For our purposes, we will replace $\alpha$ with a ``step matrix" $A$.
\begin{proof}[Proof of Theorem \ref{main}]
Recalling that $\yk_i = \tk^\T x_i$, we can rewrite equation (\ref{grad}):
\begin{align}
\tf - A\sum_{i=1}^k (\tf^\T x_i - \yk_i)x_i &= \tf - A\sum_{i=1}^k (\tf^\T x_i - \tk^\T x_i)x_i  \nonumber \\
\label{rewrite} &= \tf - A\left(\sum_{i=1}^k x_i x_i^\T\right)(\tf - \tk).
\end{align}
Let $B = \sum_{i=1}^k x_i x_i^\T.$ Note that $\textrm{range}(B) = \textrm{span}\{x_1,\ldots,x_k\} \deq V_k$. Due to the form that $B$ has, we can efficiently compute its eigendecomposition $B = V\Lambda V^\T$, where $\Lambda = \textrm{diag}(\lambda_1,\ldots,\lambda_k, 0,\ldots, 0)$, $\lambda_1,\ldots,\lambda_k$ are the nonzero eigenvalues of $B$, and $VV^\T = I$. We then define 
\begin{equation}\label{stepdef}
    A = V\Lambda^\dagger V^\T, \hspace{.25in} \Lambda^\dagger = \textrm{diag}(\lambda_1^{-1}, \ldots, \lambda_k^{-1}, 0, \ldots, 0).
\end{equation}
This choice of $A$ gives us $AB = \sum_{i=1}^k v_i v_i^\T = \textrm{proj}_{V_k}$, and therefore the update (\ref{rewrite}) is equivalent to
\begin{equation}\label{proj}
    \tf + \textrm{proj}_{V_k}(\tk - \tf).
\end{equation}
This establishes the first claim in Theorem \ref{main}. It remains to perform the computational cost calculation. We analyze the computational cost of the algorithm by breaking it down into several submodules.

\noindent\textbf{Step 1: Computing $\yk_i$, $i=1,\ldots,k$}

By Theorem \ref{lko}, this step can be accomplished in $O(k^3)$ time.

\noindent\textbf{Step 2: Finding the eigendecomposition of $\sum_{i=1}^k x_ix_i^\T$}

We will show that this step can be completed in $O(k^2 d)$ time. We compute the eigendecomposition of $B\equiv \sum_{i=1}^k x_ix_i^\T$ as follows.
\begin{enumerate}[I.]
\item Perform Gram-Schmidt on $x_1,\ldots,x_k$ to recover $u_1,\ldots,u_k$ and coefficients $c_{ij}$. computational cost: $O(k^2d)$.
\begin{enumerate}
    \item In the $i$-th step, we set $w_i = x_i - ((x_i^\T u_1)u_1 + \cdots + (x_i^\T u_{i-1})u_{i-1}),$ followed by $u_i = w_i / \lVert w_i \rVert$. Naively computing the dot products, scalar-vector products, and vector sums for step $i$ takes $O(id)$ time. Summing over the steps, the total time to perform Gram-Schmidt is $\sum_{i=1}^k O(id) = O(k^2d)$.
    \item From the $i$-th step of Gram-Schmidt, we see that
    \begin{align*}
        x_i &= (x_i^\T u_1)u_1 + \cdots + (x_i^\T u_{i-1})u_{i-1} + \lVert w_i \rVert u_i \\
        \therefore c_{ij} &= \begin{cases}x_i^\T u_j, & 1\leq j < i \\ \lVert w_i \rVert, & j = i \\ 0, & j > i \end{cases}
    \end{align*}
    We can store these coefficients as we compute them during the Gram-Schmidt procedure without increasing the asymptotic time complexity of this step.
\end{enumerate}
\item Eigendecompose the $k\times k$ matrix $C = \sum_{i=1}^k c_ic_i^\T$ and recover the eigendecomposition of $B$. computational cost: $O(k^2 d)$.
\begin{enumerate}
    \item We claim that the first $k$ eigenvalues of $B$ are identical to the eigenvalues of $C$, and that the eigenvectors of $B$ can easily be recovered from the eigenvectors of $C$. In particular, if $a_1,\ldots,a_k\in\mathbb{R}^k$ are the eigenvectors of $C$, then $v_i = a_{i1}u_1 + \ldots + a_{ik}u_k$ is the $i$-th eigenvector of $B$.
    
    To see this, note that $R(B) = \textrm{span}\{x_1,\ldots,x_k\}$, so any eigenvector for a nonzero eigenvalue of $B$ must be in the span of the $x_i$. Since $u_1,\ldots,u_k$ have the same span as the $x_i$, if $v$ is an eigenvector for $B$ with nonzero eigenvalue $\lambda$, we can write $v=b_1u_1+\cdots+b_ku_k$. Let $b=(b_1,\ldots,b_k)^\T\in\mathbb{R}^k$. We can also rewrite $B = \sum_{i=1}^k x_i x_i^\T = \sum_{i, j, \ell=1}^k c_{ij} c_{i\ell} u_j u_\ell^\T$. Combining these facts yields
    \begin{align*}
        Bv &= \sum_{i,j,\ell=1}^k b_\ell c_{i\ell} c_{ij}u_j \\
        &= \sum_{i,j=1}^k (c_i^\T b) c_{ij}u_j \\
        &= \lambda b_1 u_1 + \cdots + \lambda b_k u_k.
    \end{align*}
    Since the $u_j$s are linearly independent, we can equate coefficients. Doing so shows that $\lambda b_j = \sum_{i=1}^k (c_i^\T b)c_{ij}$ for all $j=1,\ldots,k$. Vectorizing these equations, we have that $$Cb = \sum_{i=1}^k c_i (c_i^\T b) = \lambda b.$$ This chain of equalities holds in reverse order as well, so we conclude that $v$ is an eigenvector for $B$ with nonzero eigenvalue $\lambda$ iff $b$ is an eigenvector for $C$ with eigenvalue $\lambda$. Since we know that the remaining eigenvalues of $B$ are $0$, it suffices to find an eigendecomposition of $C$. Forming $C$ takes $O(k^3)$ time, and finding its eigendecomposition can be done (approximately) in $O(k^3)$ time, see \citep{eigen}. Finally, converting each eigenvector $a_i$ for $C$ into an eigenvector for $B$ takes $O(kd)$ time (we set $v_i = a_{i1}u_1 + \cdots + a_{ik}u_k$), so converting all $k$ of them takes $O(k^2d)$ time.
    
    \item Since we know $B$ is rank $k$, the remaining eigenvalues are $0$ and any orthonormal extension of the orthonormal eigenvectors $v_1,\ldots, v_k$ computed in step 2 will suffice to complete an orthonormal basis of eigenvectors for $\mathbb{R}^d$. Let $v_{k+1}, \ldots, v_d$ be any such extension. This gives us a complete orthonormal basis of eigenvectors $v_1,\ldots,v_d$ for $\mathbb{R}^d$ with associated eigenvalues $\lambda_1\geq \lambda_2\geq \cdots \geq \lambda_k > \lambda_{k+1}=\ldots=\lambda_d = 0$. We now define $A$ as in equation (\ref{stepdef}). Observe that since $(\Lambda^\dagger)_{ii} = 0$ for $i>k$, we can compute $A$ without needing to know the values of $v_{k+1},\ldots,v_d$:
    \begin{equation}\label{step}
        A = \sum_{i=1}^k \lambda_i^{-1} v_i v_i^\T.
    \end{equation}
\end{enumerate}
\end{enumerate}

\noindent\textbf{Step 3: Performing the update}

We will show that this step can be completed in $O(kd)$ time. Recall the form of the projective residual update: $$\tr = \tf - A\sum_{i=1}^k (\tf^\T x_i - \yhat_i')x_i.$$
\begin{enumerate}[I.]
    \item Form the vector $\nabla L = \sum_{i=1}^k (\tf^\T x_i - \yhat_i')x_i.$  (This is the gradient of the loss on the synthetic datapoints.) computational cost: $O(kd)$.
    \item Compute the step $A\nabla L$. computational cost: $O(kd)$.
    \begin{enumerate}
        \item Rather than performing the computationally expensive operations of forming the matrix $A$, then doing a $d\times d$ matrix-vector multiplication, we use the special form of $A$. Namely, we have
        \begin{align}
            A\nabla L &= \left(\sum_{i=1}^k \lambda_i^{-1} v_iv_i^\T\right)s \nonumber \\
            &= \sum_{i=1}^k (\lambda_i^{-1}(v_i^\T s))v_i. \label{fastprod}
        \end{align}
        \item Each term in the summand (\ref{fastprod}) can be computed in $O(d)$ time, so we can compute the entire sum in $O(kd)$ time.
    \end{enumerate}
    \item Update $\tr = \tf - A\nabla L$. computational cost: $O(d)$.
\end{enumerate}
Since we have assumed $k\leq d$, the total computational cost of the algorithm is therefore $O(k^3) + O(k^2 d) + O(kd) = O(k^2 d)$ as desired.
\end{proof}
Note that the crucial step of computing the exact leave-$k$-out predicted $y$-values may vary depending on the the specific instance of least squares we found ourselves in (e.g. with or without regularization or weighting, see Appendix \ref{gen_to_logistic}), but the rest of the algorithm remains exactly the same.

\section{Performance analysis for outlier removal} \label{sec: outlier}
In this section we prove Theorem \ref{infbad}. We also quantify the behavior of the true step $\tf - \tl$ as the outlier size $\lambda$ grows.

\begin{proposition}\label{asymptotic_baseline}
Let $\Df$ be as in Theorem \ref{infbad}. As $\lambda\rightarrow\infty$, $\tf - \tl\rightarrow C\hat{\Sigma}^{-1} x_1$, where $\hat{\Sigma}$ is the empirical covariance matrix for the dataset $\Dl$ and $C$ is a (data-dependent) scalar constant.
\end{proposition}
\begin{proof}
Departing slightly from the notation in section 2, let $X$ and $Y$ denote the feature matrix and response vector, respectively, for the dataset $\Dl$. The exact values of $\tf$ and $\tl$ are then given by
\begin{align*}
    \tl &= (X^\T X)^{-1}XY \\
    \tf &= (X^\T X + \lambda^2 x_1x_x^\T)^{-1} (XY + \lambda^2 y_1 x_1).
\end{align*}
We can expand the expression for $\tf$ with the Sherman-Morrison formula:
\begin{equation}\label{full}
\tf = \left[ (X^\T X)^{-1} - \frac{\lambda^2 (X^\T X)^{-1}x_1 x_1^\T (X^\T X)^{-1}}{1+\lambda^2 x_1^\T (X^\T X)^{-1} x_1} \right]\cdot (XY + \lambda^2 y_1 x_1). 
\end{equation}
From equation (\ref{full}), we see that the actual step is
\begin{align}
\tf - \tl \nonumber &= \lambda^2 y_1 (X^\T X)^{-1}x_1 -\frac{\lambda^2 (X^\T X)^{-1}x_1 x_1^\T (X^\T X)^{-1}}{1+\lambda^2 x_1^\T (X^\T X)^{-1} x_1}(XY + \lambda^2 y_1 x_1) \nonumber \\[5pt]
&= (X^\T X)^{-1}\left(y_1\underbrace{\lambda^2\left[ I - \frac{\lambda^2 x_1 x_1^\T (X^\T X)^{-1}}{1+\lambda^2 x_1^\T (X^\T X)^{-1} x_1}\right]x_1}_{\textrm{(I)}} - \underbrace{\frac{\lambda^2 x_1x_1^\T (X^\T X)^{-1}}{1 + \lambda^2 x_1^\T (X^\T X)^{-1}x_1}XY}_{\textrm{(II)}}\right). \label{truestep}
\end{align}
Let us analyze the behavior of the terms (I) and (II) in equation (\ref{truestep}) as $\lambda\rightarrow\infty$. Term (II) is straightforward: the $\lambda^2$ terms dominate both the numerator and the denominator, so we have
$$\textrm{(II)}\longrightarrow \frac{x_1x_1^\T (X^\T X)^{-1}XY}{x_1^\T (X^\T X)^{-1} x_1} = \frac{x_1^\T (X^\T X)^{-1}XY}{x_1^\T (X^\T X)^{-1}x_1}x_1 \hspace{.25in} \textrm{as } \lambda\rightarrow\infty.$$ The term (I) is slightly more delicate, since the first-order behavior of (I) without multiplication by $\lambda^2$ tends to 0; however, the multiplication by $\lambda^2$ means that this term does not vanish. Observing that (I) can be rewritten as $$\lambda^2\left[1 - \frac{\lambda^2 x_1^\T (X^\T X)^{-1} x_1}{1 + \lambda^2 x_1^\T (X^\T X)^{-1}x_1}\right]x_1,$$ we have reduced our analysis of (I) to determining the leading order behavior of a function of the form \begin{equation}\label{I}f(\lambda)\equiv \lambda^2 \left[1-\frac{c\lambda^2}{1 + c\lambda^2}\right].\end{equation} (In our case, $c = x_1^\T (X^\T X)^{-1} x_1.$) A Taylor expansion of (\ref{I}) shows that $f(\lambda) = c^{-1} + O(\lambda^{-2})$, and thus we have
$$\textrm{(I)}\longrightarrow \frac{x_1}{x_1^\T (X^\T X)^{-1} x_1} \hspace{.25in} \textrm{as } \lambda\rightarrow\infty.$$ Substituting the limits of (I) and (II) into equation (\ref{truestep}), we see that
\begin{align}
    \tf - \tl \nonumber &\rightarrow (X^\T X)^{-1}\left[\frac{y_1}{x_1^\T (X^\T X)^{-1} x_1} x_1 - \frac{x_1^\T (X^\T X)^{-1}XY}{x_1^\T (X^\T X)^{-1} x_1}x_1\right] \nonumber\\[5pt]
    &= \underbrace{\frac{y_1 - x_1^\T (X^\T X)^{-1} XY}{x_1^\T (X^\T X)^{-1} x_1}}_{C'}(X^\T X)^{-1} x_1. \label{almost}
\end{align}
The result follows by multiplying and dividing (\ref{almost}) by a factor of $n$ (so $C=C'/n$ and the other factor of $n$ gets pulled into $(X^\T X)^{-1}$ to yield $\hat{\Sigma}^{-1}$).
\end{proof}

\begin{proof}[Proof of Theorem \ref{infbad}]
We will analyze $\ti - \tl$ and show that the limit of this difference is the same as that of $\tf - \tl$ as $\lambda\rightarrow\infty$; it immediately follows that $\ti\rightarrow\tf$. By the exactness of the Newton update for linear regression, we have
$$\tl = \tf + (X^\T X)^{-1}\lambda^2(\tf^\T x_1 - y_1)x_1.$$
By definition, the influence parameters are given by
$$\ti = \tf + (X^\T X + \lambda^2 x_1x_1^\T)^{-1}\lambda^2(\tf^\T x_1 - y_1)x_1.$$
Subtracting these two expressions yields
\begin{equation}\label{inf_analysis1}
    \ti - \tl = \lambda^2(\tf^\T x_1 - y_1) \cdot [(X^\T X + \lambda^2 x_1x_1^\T)^{-1} - (X^\T X)^{-1}] x_1. 
\end{equation}
We analyze the terms in the RHS of (\ref{inf_analysis1}) separately.

First, note that by the Sherman-Morrison formula, we have
\begin{align}
    [(X^\T X + \lambda^2 x_1x_1^\T)^{-1} - (X^\T X)^{-1}]x_1 &= \frac{-\lambda^2 (X^\T X)^{-1} x_1x_1^\T (X^\T X)^{-1}}{1 + \lambda^2 x_1^\T (X^\T X)^{-1} x_1} x_1\nonumber \\[5pt]
    &= \frac{-\lambda^2 x_1^\T (X^\T X)^{-1} x_1}{1+\lambda^2 x_1^\T (X^\T X)^{-1}x_1}(X^\T X)^{-1}x_1 \label{almost_term1} \\[5pt]
    &\rightarrow -(X^\T X)^{-1}x_1.\label{term1}
\end{align}
Equation (\ref{term1}) follows since the numerator and denominator of (\ref{almost_term1}) have the same leading order behavior in $\lambda$.

Next, we analyze the term $\tf^\T x_1 - y_1$. We begin by substituting the expression for $\tf$ and once more applying the Sherman-Morrison formula:
\begin{align}
    x_1^\T \tf - y_1 &= x_1^\T (X^\T X + \lambda^2 x_1x_1^\T)^{-1}(XY + \lambda^2 x_1 y_1) - y_1 \nonumber \\[5pt]
    &= x_1^\T \left[ (X^\T X)^{-1} - \frac{\lambda^2 (X^\T X)^{-1}x_1 x_1^\T (X^\T X)^{-1}}{1 + \lambda^2 x_1^\T (X^\T X)^{-1} x_1}\right](XY + \lambda^2 x_1 y_1) - y_1 \nonumber \\[5pt]
    &= x_1^\T (X^\T X)^{-1}\left[\underbrace{\left(I - \frac{\lambda^2 x_1 x_1^\T (X^\T X)^{-1}}{1 + \lambda^2 x_1^\T (X^\T X)^{-1} x_1}\right)XY}_{(i)} + y_1\underbrace{\left(\lambda^2\left[ I - \frac{\lambda^2 x_1 x_1^\T (X^\T X)^{-1}}{1 + \lambda^2 x_1^\T (X^\T X)^{-1} x_1}\right]x_1\right)}_{(ii)}\right] - y_1 \label{outlier_residual1}
\end{align}
We rearrange $(i)$ and $(ii)$ and then Taylor expand:
\begin{align*}
(i) &= XY - \frac{\lambda^2 x_1^\T (X^\T X)^{-1} XY}{1+\lambda^2 x_1^\T (X^\T X)^{-1} x_1} x_1 \nonumber \\[5pt]
&= XY - \frac{x_1^\T (X^\T X)^{-1} XY}{x_1^\T (X^\T X)^{-1} x_1}x_1 - \frac{x_1^\T (X^\T X)^{-1} XY}{(x_1^\T (X^\T X)^{-1} x_1)^2}\lambda^{-2}x_1 + O(\lambda^{-4})\\[5pt]
(ii) &= \lambda^2\left[1-\frac{\lambda^2 x_1^\T (X^\T X)^{-1} x_1}{1+\lambda^2 x_1^\T (X^\T X)^{-1} x_1}\right]x_1 \\[5pt]
&= \frac{x_1}{x_1^\T (X^\T X)^{-1} x_1} - \frac{\lambda^{-2}x_1}{(x_1^\T (X^\T X)^{-1} x_1)^2} + O(\lambda^{-4})
\end{align*}
Substituting these equations into equation (\ref{outlier_residual1}) yields
\begin{align}
x_1^\T \tf - y_1 &= x_1^\T (X^\T X)^{-1}\bigg[XY - \frac{x_1^\T (X^\T X)^{-1} XY}{x_1^\T (X^\T X)^{-1} x_1}x_1 + \frac{y_1x_1}{x_1^\T (X^\T X)^{-1} x_1} \nonumber \\[5pt]
&\hspace{1in} - \frac{x_1^\T (X^\T X)^{-1} XY}{(x_1^\T (X^\T X)^{-1} x_1)^2}\lambda^{-2}x_1 - \frac{\lambda^{-2}y_1x_1}{(x_1^\T (X^\T X)^{-1} x_1)^2}\bigg] - y_1 + O(\lambda^{-4})\nonumber\\[10pt]
&= \left(x_1^\T (X^\T X)^{-1} XY - \frac{x_1^\T (X^\T X)^{-1} XY}{x_1^\T (X^\T X)^{-1} x_1} x_1^\T (X^\T X)^{-1} x_1 + \frac{y_1x_1^\T (X^\T X)^{-1} x_1}{x_1^\T (X^\T X)^{-1} x_1} - y_1\right) \nonumber \\[5pt]
&\hspace{.25in} + \left(\frac{x_1^\T (X^\T X)^{-1} XY - y_1}{x_1^\T (X^\T X)^{-1} x_1}\right)\lambda^{-2} + O(\lambda^{-4})\nonumber\\[10pt]
&= \frac{x_1^\T (X^\T X)^{-1} XY - y_1}{x_1^\T (X^\T X)^{-1} x_1}\lambda^{-2} + O(\lambda^{-4}).\label{outlier_residual2}
\end{align}

Finally, we substitute the expressions from equations (\ref{term1}) and (\ref{outlier_residual2}) into (\ref{inf_analysis1}) to obtain
\begin{align*}
    \ti - \tl &= \left(-\frac{(X^\T X)^{-1} x_1x_1^\T (X^\T X)^{-1}}{x_1^\T (X^\T X)^{-1}x_1} + O(\lambda^{-2})\right)\lambda^2\left(\frac{x_1^\T (X^\T X)^{-1} XY - y_1}{x_1^\T (X^\T X)^{-1} x_1}\lambda^{-2} + O(\lambda^{-4}) \right)x_1 \\[5pt]
    &= \frac{y_1 - x_1^\T (X^\T X)^{-1} XY}{x_1^\T (X^\T X)^{-1} x_1}(X^\T X)^{-1}x_1 + O(\lambda^{-2}).
\end{align*}
Note that this has the same limiting value as $\tf - \tl$ as $\lambda\rightarrow\infty$ (see equation (\ref{almost})) and we are done.
\end{proof}

\section{Proof of Theorem \ref{lko}} \label{proof: lko}
We prove Theorem \ref{lko} for the case of ordinary least squares. We generalize this logic to weighted, ridge regularized least squares in Appendix \ref{appendix: generalized lko}.

\begin{proof}[Proof of Theorem \ref{lko}]
We make use of the analytic form of the parameters for least squares linear regression. Given a dataset $\{(x_i, y_i)\}_{i=1}^n$, we have $\tf = (X^\T X)^{-1} XY$, with $X, Y$ defined as in section 2. The predictions for the fitted model on the dataset are then given by \begin{equation}\label{hat}\hat{Y} = X\tf = \underbrace{(X(X^\T X)^{-1}X^\T)}_H Y,\end{equation} where $H = X(X^\T X)^{-1} X^\T$ is the so-called hat matrix. As previously mentioned, we assume that we already have access to $H$ after the model has been trained on the full dataset.

Next, observe that $$\tk = \argmint \left[\sum_{i=1}^k (\theta^\T x_i - \yk_i)^2 + \sum_{i=k+1}^n (\theta^\T x_i - y_i)^2\right]$$ since $\tk$ minimizes both sums individually. It follows from equation (\ref{hat}) that $HY' = \hat{Y}_{\setminus k}$, where $Y' = (\yk_1, \ldots, \yk_k, y_{k+1}, \ldots, y_n)^\T$ and $\hat{Y}_{\setminus k} = (\yk_1,\ldots,\yk_n)^\T$.

This relation $HY' = \hat{Y}_{\setminus k}$ allows us to derive a system of linear equations between $\yk_i$ for $i=1,\ldots,k$. Namely, if we define $r_i = y_i - \yhat_i$, $r = (r_1,\ldots,r_k)^\T$, $r^{\setminus k}_i = y_i - \yk_i$, and $r^{\setminus k}=(r^{\setminus k}_1, \ldots, r^{\setminus k}_k)^\T$, we have
\begin{equation}\label{res1}
    r^{\setminus k}_i = \frac{r_i + \sum_{j\neq i} h_{ij} r^{\setminus k}_j}{1-h_{ii}},
\end{equation}
where $h_{ij}$ are the entries of $H$. Vectorizing equation (\ref{res1}) and solving yields
\begin{equation}\label{lin2}\textstyle
    r^{\setminus k} = (I-T)^{-1}(\frac{r_1}{1-h_{11}}, \: \ldots, \: \frac{r_k}{1-h_{kk}})^\T,
\end{equation}
where $T_{ij} = \textbf{1}\{i\neq j\} \frac{h_{ij}}{1-h_{jj}}$. Since this is a system of $k$ linear equations in $k$ unknowns, we can solve it in time $O(k^3)$ via simple Gaussian elimination. The values $\yk_i$ can then be easily recovered in an additional $O(k)$ time by noting that $\yk_i = y_i - r^{\setminus k}_i$.
\end{proof}

\section{Generalization of Theorem \ref{lko} to weighted, ridge regularized least squares} \label{appendix: generalized lko}
Refer to Appendix \ref{proof: lko}. We can generalize our method for computing the predictions of the LKO model to weighted least squares with ridge regularization. Let $w\succeq0\in\mathbb{R}^n$ denote a (fixed) weight vector and $\lambda\geq0$ be the regularization strength, which we require to be fixed independent of the number of samples. The weighted, regularized loss is given by
\begin{align*}
\Lf(\theta) &= \frac12\left(\sum_{i=1}^n w_i(\theta^\T x_i - y_i)^2 + \lambda \lVert \theta \rVert^2\right) \\
&= \frac12 [(X\theta - Y)^\T W (X\theta - Y) + \lambda \lVert \theta \rVert^2].
\end{align*}
The gradient is therefore
\begin{equation}\label{weightreglossgrad}\nabla \Lf(\theta) = X^\T W X\theta - X^\T W Y + \lambda \theta\end{equation}
Using equation (\ref{weightreglossgrad}), we see that $\nabla \Lf = 0$ when $$\tf =  (X^\T W X + \lambda I)^{-1}X^\T W Y.$$
Predictions are therefore given by
$$X\tf = \underbrace{X(X^\T W X+\lambda I)^{-1}X^\T W}_{H_\lambda, w} Y.$$ If we replace $H$ in equation (\ref{hat}) with $H_{\lambda, w}$, the same logic carries through. Note that the regularization strength needs to be fixed for us to use the same trick, i.e. to write
$$\tk = \argmint \sum_{i=1}^k w_i(\theta^\T x_i - \yk_i)^2 + \sum_{i=k+1}^n w_i(\theta^\T x_i - y_i)^2  + \lambda\lVert \theta \rVert^2$$
with $\yk_i = \tk^\T x_i$ the predicted $y$-value for the LKO model. In this case, we can compute the LKO prediction values efficiently ($O(k^3)$ time when we precompute $H_{\lambda, w}$). Theorem \ref{lko} therefore holds in this more general setting as well.

\section{Proof of Theorem \ref{logisticresup}}\label{gen_to_logistic}


\begin{proof}
For logistic regression, we use the loss function
\begin{equation*}
L(\theta) = \sum_{i=1}^n [y_i \log h_\theta(x_i) + (1-y_i)\log(1-h_\theta(x_i))] + \frac12 \lambda\lVert \theta \rVert^2, \label{logisticloss}
\end{equation*} where $(x_i, y_i)\in \mathbb{R}^d \times \{0,1\}$ are the data, and the classifier $h_\theta(x)$ is given by $$h_\theta(x) = \frac{1}{1+\exp\{-\theta^\T x\}}.$$
We compute the gradient and Hessian of the loss:
\begin{equation}\label{logisticgrad}\nabla L(\theta) = \sum_{i=1}^n (h_\theta(x_i) - y_i)x_i + \lambda \theta = \bar{X}^\T (\bar{h}_\theta - \bar{Y}) + \lambda \theta\end{equation}
\begin{equation}\label{logistichess}\nabla^2 L(\theta) = \sum_{i=1}^n h_\theta(x_i)(1-h_\theta(x_i)) x_ix_i^\T + \lambda I = \bar{X}^\T \bar{S}_\theta \bar{X} + \lambda I,\end{equation}
where $\bar{X}\in\mathbb{R}^{n\times d}$ is the data matrix whose rows are $x_i^\T$, $\bar{h}_\theta\in\mathbb{R}^n$ is the vector of model predictions, $\bar{Y}\in\mathbb{R}^n$ is the vector of labels, and $\bar{S}_\theta = \textrm{diag}(\{h_\theta(x_i)(1-h_\theta(x_i)\}_{i=1}^n)\in\R^{n\times n}$. Using these formulas, we can compute a Newton step for the LKO loss when we start at the minimizer $\tf$ for the full loss.

Now let $X = (x_{k+1} \: \cdots x_n)^\T\in \mathbb{R}^{(n-k)\times d}$, $Y = (y_{k+1},\ldots,y_n)^\T\in\R^{n-k}$, $h_{\tf} = (h_{\tf}(x_{k+1}),\ldots,h_{\tf}(x_n))^\T\in\R^{n-k}$,  $S_{\tf}=\textrm{diag}(\{h_{\tf}(x_i)(1-h_{\tf}(x_i)\}_{i=k+1}^n)$ be the LKO quantities corresponding to the terms defined above. By definition, we have
\begin{align}
    \theta^{\textrm{Newton}} &= \tf-[\nabla^2 \Lk(\tf)]^{-1} \nabla L_{LKO}(\tf) \nonumber\\
    &= \tf + (X^\T S_{\tf} X + \lambda I)^{-1} (X^\T (Y - h_{\tf}) - \lambda \tf) \nonumber\\
    &= (X^\T S_{\tf} X + \lambda I)^{-1}X^\T S_{\tf}( X\tf + S_{\tf}^{-1}(Y - h_{\tf})) \nonumber\\
    &= (X^\T S_{\tf} X + \lambda I)^{-1} X^\T S_{\tf}Z, \label{logisticnewton}
\end{align}
where $Z \equiv X\tf + S_{\tf}^{-1}(Y - h_{\tf})$. Observe that equation (\ref{logisticnewton}) is the solution to the LKO weighted least squares problem
\begin{equation}\label{logisticls}\min_\theta \sum_{i=k+1}^n h_{\tf}(x_i)(1-h_{\tf}(x_i))(\theta^\T x_i - z_i)^2 + \lambda \lVert \theta \rVert^2,\end{equation}
where $z_i$ is the $i$-th component of $\bar{Z} \equiv \bar{X}\tf + \bar{S}_{\tf}^{-1}(\bar{Y}-\bar{h}_{\tf})$. By adapting the PRU to this situation, we can compute a fast approximation to the Newton step.

We can compute the vector $\bar{Z} \equiv \bar{X}\tf + \bar{S}_{\tf}^{-1}(\bar{Y} - \bar{h}_{\tf})$, as well as the matrix $H_{\lambda, \bar{h}_{\tf}} \equiv \bar{X}(\bar{X}^\T \bar{S}_{\tf}\bar{X} + \lambda I)^{-1} \bar{X}^\T \bar{S}_{\tf}$, offline. Observe that $H_{\lambda, \bar{h}_{\tf}}$ is the hat matrix for the ``full" least squares problem 
\begin{equation}\label{full_log_ls}
    \min_\theta \sum_{i=1}^n h_{\tf}(x_i)(1-h_{\tf}(x_i))(\theta^\T x_i - z_i)^2 + \lambda \lVert \theta \rVert^2.
\end{equation}
For consistency with the rest of the paper, let $\tk$ be the exact solution to (\ref{logisticls}) (so $\tk = \theta^{\textrm{Newton}}$). By the result of Theorem \ref{lko}, we can compute the LKO model predictions $\zk_i \equiv \tk^\T x_i$, $i=1,\ldots,k$ in $O(k^3)$ time. Observe that the gradient of the (unregularized, unweighted, quadratic) loss on the synthetic points $(x_i, \zk_i)$ is
\begin{equation}\label{synth_log_loss}
    \sum_{i=1}^k (\tf^x_i - \zk)x_i = \left(\sum_{i=1}^k x_i x_i^\T\right)(\tf - \tk).
\end{equation}
We are now in a setting exactly analogous to equation (\ref{rewrite}), even though $\tf$ was the minimizer for the original \emph{cross-entropy} objective rather than (\ref{full_log_ls}). By mimicking the proof of Theorem \ref{main} from this point, we can derive the exact same results. Namely, the step taken by the projective residual update is equal to $\textrm{proj}_{\textrm{span}(x_1,\ldots,x_k)}(\tk-\tf)$. By definition of $\tk$ and of the Newton step, it follows that $\tk-\tf=\Delta_{\textrm{Newton}}$. Combining these two facts yields the statement of Theorem \ref{logisticresup}. The computational cost calculation is identical to the calculation in Theorem \ref{main}.
\end{proof}

\section{Synthetic data construction} \label{sec: data construction}
We first generate a matrix of $n$ $d$-dimensional covariates $X\in\R^{n\times d}$; we do this by drawing the rows $x_i^\T$ of $X$ according to $x_i \iid N(0, \Sigma)$, where $\Sigma$ is randomly selected via \texttt{sklearn.datasets.make\_spd\_matrix} \citep{sklearn}. Once $X$ is generated, the response vector $Y\in\R^n$ is generated by randomly selecting a (fixed) ``true" underlying parameter $\theta^*\in\R^d$, and setting $Y = X\theta^* + \varepsilon$, where $\varepsilon \sim N(0, \sigma^2 I_n)$ is the error vector. For our experiments, we set the noise level $\sigma^2 = 1$; for reasonable values of $\sigma^2$ this parameter does not play a large role in the outcome of the experiments. For all of the synthetic experiments, when deleting a group of size $k$, we always assume that it is the first $k$ datapoints which are being deleted. (That is, we delete the datapoints specified by the first $k$ rows of $X$ and the first $k$ entries of $Y$.) For all of the synthetic datasets, we take $n=10d$.

For the runtime experiment, no modifications are made to the general setup. We vary the dimension $d$ between $d=1000$ and $d=3000$ and the group size $k$ between $k=1$ and $k=100$.

For the $L^2$ experiment, we first construct $\tilde{X}$ and $\tilde{Y}$ according to the general procedure above. We then obtain the data $X, Y$ by multiplying the first $k$ rows of $\tilde{X}$ and the first $k$ entries of $\tilde{Y}$ (that is, the points which will eventually be deleted) by a factor $\lambda$ to demonstrate the effectiveness of each method at removing outlier datapoints. This is the setting described in section \ref{performance}.

The modifications for the FIT experiment are slightly more involved. We construct sparse data with three key properties: (1) only the deleted feature vectors $x_i$, $i=1,\ldots,k$ have nonzero $d$-th entry (this is the ``injected feature"); (2) the deleted feature vectors all lie on the same low-dimensional subspace; (3) the response for the deleted points is perfectly correlated with the special feature. The exact steps for this procedure are as follows:
\begin{enumerate}
    \item Construct $\tilde{X}$ according to the general procedure. (Pick a random covariance $\Sigma$ and draw the rows $x_i^\T$ of $\tilde{X}$ according to $x_i\iid N(0, \Sigma)$.)
    \item The ``injected feature" will be the last ($d$-th) entry of each vector $x_i$. Since only the group being removed has the injected feature, we set the last entry in rows $k+1$ to $n$ of $\tilde{X}$ equal to 0; the first $k$ rows keep their original final entry.
    \item Sparsify $\tilde{X}$ so that it has a fraction of approximately $p$ nonzero entries. Let $\tilde{X}[i, j]$ denote the $(i, j)$-th entry of $\tilde{X}$.
    \begin{enumerate}
        \item Sparsify the first $k$ rows of $\tilde{X}$ simultaneously: for each $j=1,\ldots,d-1$, set $\tilde{X}[i, j] = 0$ for all $i=1,\ldots,k$ with probability $1-p$.
        \item Sparsify the remaining entries of $\tilde{X}$: for each $i=k+1,\ldots, n$ and $j=1,\ldots,d-1$, set $X[i, j] = 0$ with probability $1-p$.
    \end{enumerate}
    \item Let $X$ be the matrix resulting from performing operations 1-3 on $\tilde{X}$. Set $\tilde{Y}$ according to the general procedure: $\tilde{Y} = X\theta^* + \varepsilon$.
    \item Let $\tilde{Y}[i]$ denote the $i$-th entry of $\tilde{Y}$. For $i=1,\ldots, k$, set $\tilde{Y}[i] = w_*  X[i, d]$, where $w_*$ is the pre-specified ``true" weight of the injected feature.
\end{enumerate}
For the sparse logistic regression experiment, the matrix of covariates $X$ is generated according to the procedure above. The labels $Y$ are then generated so that the logistic model is well-specified, as described in section \ref{sec: logistic regression}.

To generate data for the FIT for logistic regression, the procedure is the same as the one outlined for linear regression above (and as outlined in Section \ref{metrics}) with some minor changes. WLOG assume that the points to be deleted are $(X[i, :], Y[i])_{i=1}^k$. We require the following:
\begin{enumerate}
    \item The $k$ points to be deleted all belong to the positive class, i.e. $Y[i] = 1$, $i=1,\ldots, k$.
    \item The $k$ points to be deleted are classified correctly by the full model, i.e. $X[i,\, :]^\T \tf > 0$, $i=1,\ldots,k$. ($X[i,\, :]$ denotes the $i$-th row of the data matrix $X$.)
    \item The $k$ points to be deleted have injected feature equal to 1, while the points that remain all have injected feature equal to 0. That is, $X[i, d] = 1$ for $i = 1,\ldots, k$ and $X[i, d] = 0$ for all $i > k$.
\end{enumerate}

\section{Baseline values for synthetic linear regression experiments}
All of the experimental results in the main body of the paper are given relative to an absolute baseline value. In Tables \ref{runtime_abs}, \ref{weight_abs}, and \ref{l2_abs}, we report the medians of the absolute baseline values to which we are comparing for each of the three synthetic experiments (runtime, $L^2$, and feature injection, respectively) for linear regression. The baseline values follow the trends we would expect. In particular, the runtimes increase sharply as the dimension increases and slowly as the group size increases (Table \ref{runtime_abs}); the unique feature weight originally learned by the model is close to its true value, 10 (Table \ref{weight_abs}); and the distance between $\tf$ and $\tk$ increases with the number of points removed, as well as the dissimilarity of these points to the rest of the dataset as measured by the multiplier $\lambda$ (Table \ref{l2_abs}).
\begin{table}[h]
\caption{Median exact retraining runtimes in seconds for table \ref{runtime_experiment}. The method used was a Newton step with the Sherman-Morrison formula.}
\label{runtime_abs}
\begin{center}
\begin{small}
\begin{sc}
\begin{tabular}{l|ccccc}
\toprule
& $d = 1000$ & $d = 1500$ & $d = 2000$ & $d = 2500$ & $d = 3000$\\
\midrule
$k = 1$ &
0.08 &
0.27 &
0.67 &
1.19 &
2.25 \\
\midrule
$k = 5$ &
0.08 &
0.31 &
0.67 &
1.33 &
2.22 \\
\midrule
$k = 10$ &
0.08 &
0.31 &
0.63 &
1.33 &
2.04 \\
\midrule
$k = 25$ &
0.08 &
0.32 &
0.62 &
1.36 &
2.06 \\
\midrule
$k = 50$ &
0.09 &
0.33 &
0.64 &
1.40 &
2.11 \\
\bottomrule
\end{tabular}
\end{sc}
\end{small}
\end{center}
\end{table}

\begin{table}[!htb]
\begin{minipage}{.47\linewidth}
    \centering
    \caption{Median baseline weights on injected feature for table \ref{synthetic_weight}.}
    \label{weight_abs}
\begin{tabular}{l|cccc}
\toprule
& $p = 0.25$ & $0.1$ & $0.05$\\
\midrule
$k = 5$ &
8.97 &
10.61 &
11.03 \\
\midrule
$k = 50$ &
10.23 &
9.73 &
10.10 \\
\midrule
$k = 100$ &
9.51 &
9.99 &
10.01 \\
\bottomrule
\end{tabular}
\end{minipage}
\hfill
\begin{minipage}{.47\linewidth}
    \centering
    \caption{Median baseline $L^2$ parameter distances for table \ref{synthetic_l2}.}
    \label{l2_abs}
\begin{tabular}{l|ccc}
\toprule
& $\lambda = 1$ & $\lambda = 10$ & $\lambda = 100$ \\
\midrule
$k = 5$ &
0.018 &
0.175 &
0.192 \\
\midrule
$k = 50$ &
0.057 &
0.523 &
0.572 \\
\midrule
$k = 100$ &
0.082 &
0.761 &
0.842 \\
\bottomrule
\end{tabular}
\end{minipage}
\end{table}

\section{Detailed experimental results for logistic regression}
Here we give the complete results---inlcuding the results for Newton's method, as well as the IQR for each setting---for the logistic regression experiments. As explained in appendix \ref{gen_to_logistic}, the logistic PRU computes a projection of the Newton step onto a lower-dimensional subspace. (In fact, for both linear and logistic regression, the PRU computes a projection of the Newton step. It just happens that for linear regression, the Newton step is exact, while this is no longer the case for logistic regression.) As a result, retraining via Newton's method is more accurate than retraining via the PRU. The PRU's advantage lies in its combination of accuracy and speed. While slightly less accurate than Newton's method, the PRU can be up to thousands of time faster. Indeed, since the computational cost of Newton's method for logistic regression is the same as the computational cost of exact retraining for linear regression, the PRU has the same favorable runtime comparisons as in Table \ref{runtime_experiment}.

\begin{table}[H]
\caption{Complete results for the sparse logistic FIT. For larger group sizes and sparse data, the PRU is able to completely remove the injected feature. With any strictly positive regularization, Newton's method will completely remove the injected feature, but its computational cost is vastly slower than that of the PRU (see Table \ref{runtime_experiment}).}
\begin{center}
\begin{small}
\begin{sc}
\begin{tabular}{l|ccc}
\toprule
              & $p=0.5$            & $p=0.1$            & $p=0.05$           \\
\midrule
$k=25$ (inf)  & \textbf{0.82 (0.79 - 0.82)} & 0.76 (0.73 - 0.80) & 0.78 (0.77 - 0.79) \\
$k=25$ (pru)  & 0.86 (0.83 - 0.88) & \textbf{0.69 (0.64 - 0.70)} & \textbf{0.44 (0.40 - 0.50)} \\
$k=25$ (nwt)  & 0.0 (0.0 - 0.0)    & 0.0 (0.0 - 0.0)    & 0.0 (0.0 - 0.0)    \\
\midrule
$k=50$ (inf) & 0.81 (0.78 - 0.84) & 0.82 (0.81 - 0.83) & 0.82 (0.80 - 0.84) \\
$k=50$ (pru) & 0.81 (0.78 - 0.84) & \textbf{0.48 (0.48 - 0.54)} & \textbf{0.02 (0.00 - 0.03)} \\
$k=50$ (nwt) & 0.0 (0.0 - 0.0)    & 0.0 (0.0 - 0.0)    & 0.0 (0.0 - 0.0)    \\
\midrule
$k=100$ (inf) & 0.82 (0.81 - 0.83) & 0.85 (0.83 - 0.86) & 0.84 (0.82 - 0.85) \\
$k=100$ (pru) & \textbf{0.71 (0.69 - 0.71)} & \textbf{0.00 (0.00 - 0.01)} & \textbf{0.0 (0.0 - 0.0)}    \\
$k=100$ (nwt) & 0.0 (0.0 - 0.0)    & 0.0 (0.0 - 0.0)    & 0.0 (0.0 - 0.0)    \\
\bottomrule
\end{tabular}
\end{sc}
\end{small}
\end{center}
\end{table}

\begin{table}[H]
\caption{Complete results for the sparse logistic $L^2$ experiment. For sparse data and moderate group deletion sizes, the PRU's performance surpasses the performance of the influence method. The PRU becomes nearly as accurate as Newton's method while maintaining a faster runtime.}
\begin{center}
\begin{small}
\begin{sc}
\begin{tabular}{l|ccc}
\toprule
& $p=0.5$                    & $p=0.1$                    & $p=0.05$      \\
\midrule
$k=25$ (inf)  & \textbf{0.85 (0.81 - 0.87)} & \textbf{0.78 (0.75 - 0.80)} & 0.78 (0.77 - 0.79) \\
$k=25$ (pru)  & 0.86 (0.84 - 0.87) & 0.80 (0.79 - 0.81) & \textbf{0.65 (0.63 - 0.69)} \\
$k=25$ (nwt)  & 0.08 (0.07 - 0.09) & 0.02 (0.02 - 0.03) & 0.01 (0.01 - 0.02) \\
\midrule
$k=50$ (inf) & 0.85 (0.82 - 0.86) & 0.83 (0.81 - 0.83) & 0.82 (0.80 - 0.84) \\
$k=50$ (pru) & 0.85 (0.83 - 0.86) & \textbf{0.69 (0.68 - 0.72)} & \textbf{0.20 (0.18 - 0.25)} \\
$k=50$ (nwt) & 0.08 (0.07 - 0.09) & 0.03 (0.02 - 0.03) & 0.01 (0.01 - 0.02) \\
\midrule
$k=100$ (inf) & 0.85 (0.84 - 0.85) & 0.86 (0.84 - 0.87) & 0.84 (0.82 - 0.85) \\
$k=100$ (pru) & \textbf{0.80 (0.79 - 0.81)} & \textbf{0.24 (0.21 - 0.24)} & \textbf{0.13 (0.12 - 0.14)} \\
$k=100$ (nwt) & 0.09 (0.08 - 0.09) & 0.03 (0.02 - 0.04) & 0.01 (0.01 - 0.01) \\
\bottomrule
\end{tabular}
\end{sc}
\end{small}
\end{center}
\end{table}
\end{document}